\documentclass[preprint,12pt]{elsarticle}

\usepackage[T1]{fontenc}
\usepackage[utf8]{inputenc}
\usepackage{lmodern}
\usepackage[english]{babel}
\usepackage{microtype}
\usepackage{subcaption}   
\usepackage[utf8]{inputenc}   
\usepackage[T1]{fontenc}
\usepackage{tipa} 
\usepackage{lmodern}
\usepackage{ulem} 

\usepackage{amsmath,amssymb,amsfonts,amsthm,mathtools,bm}
\numberwithin{equation}{section}

\usepackage{graphicx}
\usepackage{booktabs}
\usepackage{multirow}
\usepackage{array}

\usepackage{algorithm}
\usepackage{algorithmic}

\usepackage{comment}      
\usepackage{xcolor}
\usepackage{enumitem}

\usepackage{hyperref}
\hypersetup{
    colorlinks = true,
    linkcolor  = blue,
    citecolor  = blue,
    urlcolor   = blue
}

\usepackage{setspace}
\theoremstyle{plain}
\newtheorem{theorem}{Theorem}[section]
\newtheorem{lemma}{Lemma}[section]
\newtheorem{proposition}{Proposition}[section]
\newtheorem{corollary}{Corollary}[section]

\theoremstyle{definition}

\theoremstyle{remark}

\journal{Applied and Computational Harmonic Analysis}

\begin{document}

\begin{frontmatter}

\title{The Relativity of AGI: Distributional Axioms, Fragility, and Undecidability}

\author{Angshul Majumdar}

\begin{abstract}
We study whether Artificial General Intelligence (AGI) admits a coherent theoretical definition that supports absolute claims of existence, robustness, or self-verification. We formalize AGI axiomatically as a distributional, resource-bounded semantic predicate, indexed by a task family, a task distribution, a performance functional, and explicit resource budgets. Under this framework, we derive four classes of results. First, we show that generality is inherently relational: there is no distribution-independent notion of AGI. Second, we prove non-invariance results demonstrating that arbitrarily small perturbations of the task distribution can invalidate AGI properties via cliff sets, precluding universal robustness. Third, we establish bounded transfer guarantees, ruling out unbounded generalization across task families under finite resources. Fourth, invoking Rice-style and G\"odel--Tarski arguments, we prove that AGI is a nontrivial semantic property and therefore cannot be soundly and completely certified by any computable procedure, including procedures implemented by the agent itself. Consequently, recursive self-improvement schemes that rely on internal self-certification of AGI are ill-posed. Taken together, our results show that strong, distribution-independent claims of AGI are not false but undefined without explicit formal indexing, and that empirical progress in AI does not imply the attainability of self-certifying general intelligence.
\end{abstract}

\end{frontmatter}

\section{Introduction}\label{sec:intro}

Despite sustained public and scientific attention, the term artificial general intelligence (AGI) remains conceptually unstable. In policy discussions, industrial roadmaps, and popular commentary, AGI is routinely described as an impending milestone---a system that is ``general'' in roughly the way humans are general---yet the phrase is rarely cashed out as a precise mathematical predicate \citep{Bostrom2014,Shevlin2022,Goertzel2014,OpenAI2023Planning}. In the technical literature, evaluations often substitute breadth of benchmark performance for generality, yielding an implicit definition of AGI as ``high scores on many tasks'' \citep{Radford2019,OpenAI2023GPT4,Bommasani2021}. This substitution is convenient, but it blurs a crucial distinction: performance aggregated over curated test suites is a distributional statement about the specific task ecology encoded by those suites, not a theorem-level claim about generality in any invariant sense \citep{Amodei2016,Bender2021}.

The absence of a shared definition is not a mere terminological nuisance. It prevents falsifiable claims, encourages rhetorical goalpost shifts, and makes it difficult to separate engineering progress from statements that are, in effect, metaphysical \citep{Mitchell2019,Marcus2022,Floridi2023}. For example, one widely recurring formulation---``a system that can do any intellectual task a human can''---depends on underspecified notions of task, ability, resources, and acceptable failure \citep{RussellNorvig2021}. Another common move is to treat scaling trends as evidence of an eventual qualitative phase transition to generality \citep{Kaplan2020,Hoffmann2022}; however, without an explicit semantics for what ``general'' means, such claims are not even well-formed propositions. In short, AGI talk is currently long on aspiration and short on definitional discipline.

This paper takes the definitional question as primary. Our goal is not to deny the empirical reality of increasingly capable systems, nor to dispute the practical value of broad competence. Rather, we ask a narrower and more fundamental question: is there a coherent theoretical object corresponding to contemporary notions of AGI? To make this question answerable, we define AGI axiomatically---in the manner of a norm or a topology---as a bundle of individually reasonable properties capturing breadth, adaptivity, transfer, compositionality, and robustness. We then interpret these axioms under distributional semantics: the agent's competence is evaluated relative to an explicit task distribution, a performance functional, and a resource budget. This relational pivot forces any claim of ``AGI'' to specify the task ecology it quantifies over; without such a specification, the predicate is undefined. Under this semantics, we derive structural theorems that formalize fragility and non-invariance: even if an agent satisfies the axioms under one distribution, arbitrarily small perturbations of the task distribution can destroy the property, and transfer cannot be unbounded without commensurate shared structure. Finally, we show that any attempt by a sufficiently expressive system to internally certify its own AGI status amounts to deciding a non-trivial semantic property, linking self-certification to classical undecidability obstructions and establishing a G\"odel--Tarski-style limit on self-knowledge of ``generality'' \citep{Tarski1936,Godel1931,Rice1953}.

\section{Preliminaries and Notation}\label{sec:prelim}

This section fixes the objects and notation used throughout the paper. The intent is to make all subsequent statements---in particular the axioms of distributional AGI and the structural theorems---well-formed, unambiguous, and internally consistent.

\subsection{Spaces, Histories, and Randomness}\label{sec:prelim:spaces}
Let $\mathcal{O}$ denote the observation space and $\mathcal{A}$ the action space. In the most general setting, $\mathcal{O}$ and $\mathcal{A}$ are measurable spaces; in concrete examples they may be discrete alphabets (e.g., tokens) or Euclidean spaces (e.g., sensor vectors). Let $\mathcal{H}_t := (\mathcal{O}\times \mathcal{A})^{t-1}\times \mathcal{O}$ denote the space of length-$t$ interaction histories and let $\mathcal{H} := \bigcup_{t\ge 1}\mathcal{H}_t$ denote the set of all finite histories. For $h_t=(o_1,a_1,\ldots,a_{t-1},o_t)\in \mathcal{H}_t$, we write $h_{t:t'}$ for the subsequence $(o_t,a_t,\ldots,o_{t'})$ when needed.

All probability statements are taken with respect to an ambient probability space $(\Omega,\mathcal{F},\mathbb{P})$ that underlies the stochasticity of environments, agents (if randomized), and any exogenous randomness used to sample tasks.

\subsection{Tasks, Environments, and Performance}\label{sec:prelim:tasks}

\paragraph{Environments.}
An environment $E$ is an interactive stochastic process that induces a sequence $(o_t)_{t\ge 1}$ of observations conditioned on past interactions. Formally, $E$ specifies a family of conditional distributions
\[
E(\,\cdot \mid h_t, a_t) \in \Delta(\mathcal{O}), \qquad t\ge 1,
\]
where $\Delta(\mathcal{X})$ denotes the set of probability measures on $\mathcal{X}$. When the environment also emits an explicit reward signal, we include it in the observation $o_t$ (so $o_t$ may contain both ``state'' and ``reward'') to avoid duplicating notation.

\paragraph{Episodes and horizons.}
We consider episodic interaction with a (possibly random) horizon $T\in \mathbb{N}\cup\{\infty\}$. When a finite horizon is assumed, $T$ may be part of the task specification. For clarity we write $\mathcal{H}_{\le T} := \bigcup_{t=1}^{T}\mathcal{H}_t$.

\paragraph{Tasks.}
A task is a pair
\[
\tau := (E, U),
\]
where $E$ is an environment and $U$ is a utility (success) functional mapping a complete episode history to a real-valued score. For a finite horizon $T$, we take $U:\mathcal{H}_{\le T}\to \mathbb{R}$; for $T=\infty$, $U$ may depend on the full infinite trajectory or be defined via a limiting/discounted functional. We assume $U$ is measurable and typically bounded; when needed we assume $U\in [0,1]$ without loss of generality by rescaling.

\paragraph{Policies and induced trajectories.}
Given a policy (agent) $A$ and an environment $E$, the interaction induces a random history $H_T \sim (E \circ A)$ (composition in the standard interactive sense). We write $\mathbb{E}_{E,A}[\cdot]$ and $\mathbb{P}_{E,A}[\cdot]$ for expectation and probability under this induced law.

\paragraph{Performance metric.}
A performance metric is a functional
\[
\Pi(\tau, A; B) \in \mathbb{R},
\]
where $\tau=(E,U)$ is a task, $A$ is an agent, and $B$ is a resource budget (defined below). The canonical choice is expected utility under the induced interaction:
\[
\Pi(\tau, A; B) := \mathbb{E}_{E,A}\!\left[\,U(H_T)\,\right],
\]
subject to the constraint that $A$ respects the budget $B$. In some settings $\Pi$ may incorporate additional penalties (e.g., constraint violations, cost of tool calls), or may be a regret-like quantity. We assume $\Pi$ is bounded (e.g., in $[0,1]$) after normalization.

\paragraph{Task-level success events.}
For statements that require a binary notion of success, we introduce a task-level success event $S_\tau$ induced by the performance functional. Concretely, when $\Pi$ is normalized and a threshold $\theta\in\mathbb{R}$ is fixed, we may take
\[
S_\tau := \{\Pi(\tau,A;B)\ge \theta\}.
\]
The particular threshold used is always stated explicitly when $S_\tau$ is invoked.

\paragraph{Constraints and violation indicators.}
To incorporate constraints (e.g., safety constraints, policy constraints, specification constraints) into evaluation, we use a measurable violation indicator
\[
V_C(\tau,A;B)\in\{0,1\},
\]
where $V_C(\tau,A;B)=1$ indicates that the interaction of $A$ with task $\tau$ under budget $B$ violates the constraint family $C$. When constraints are relevant, performance statements are made jointly with bounds on $\mathbb{P}_{\tau\sim\mu}(V_C(\tau,A;B)=1)$.

\subsection{Agents, Learning, and Resource Budgets}\label{sec:prelim:agents}

\paragraph{Agents as history-dependent policies.}
An agent $A$ is a (possibly randomized) mapping from histories to action distributions:
\[
A(\,\cdot \mid h_t) \in \Delta(\mathcal{A}), \qquad h_t\in \mathcal{H}_t,\ t\ge 1.
\]
We allow $A$ to be stateful: it may maintain internal memory $m_t$ updated over time. Formally, one may view $A$ as a policy over an augmented history that includes its internal state; we suppress $m_t$ unless explicitly needed.

\paragraph{Learning and adaptation.}
To express adaptivity and transfer, we treat the agent as parameterized by an internal state/parameter $\theta$. Let $\theta_0$ denote the initial state (e.g., pretrained weights). As interaction proceeds, the agent may update $\theta$ according to some (possibly implicit) update rule driven by data $D_t$ extracted from histories:
\[
\theta_{t+1} = \mathrm{Update}(\theta_t, D_t).
\]
We write $A_{\theta}$ for the policy induced by parameter/state $\theta$, and $A_{t}$ (or $A^{(t)}$) for the agent after $t$ updates. When adaptation occurs across tasks, we write $A^{(k)}$ for the agent after $k$ tasks/episodes have been processed.

\paragraph{Update-indexed agents and comparison baselines.}
To avoid ambiguity in later axioms, we use the convention $A^{(n)}$ to denote the agent after $n$ updates as measured in the relevant context (within-task updates when proving adaptivity statements, and across-task updates when proving transfer statements). When a distinction is required between an agent that has undergone a fixed pre-exposure phase and one that has not, we write $A_{\mathrm{pre}}$ for the agent after a pre-exposure phase consisting of $K_{\mathrm{tr}}$ tasks sampled i.i.d.\ from $\mu$, and $A_{\mathrm{scratch}}$ for the corresponding agent initialized without such pre-exposure. The integer $N_{\mathrm{ad}}$ denotes a within-task update budget used to formalize adaptivity. All update counts and pre-exposure lengths are constrained by the resource budget $B$ through the requirement that the corresponding executions respect $B$.

\paragraph{Resource budgets.}
A resource budget is a tuple
\[
B := (B_{\mathrm{comp}}, B_{\mathrm{mem}}, B_{\mathrm{samp}}, B_{\mathrm{time}}, B_{\mathrm{tool}}),
\]
capturing, respectively, limits on compute, memory, number of samples/episodes, wall-clock time, and (when relevant) external tool calls or oracle queries. We say an agent respects $B$ if, on any interaction consistent with the task specification, its resource usage stays within the corresponding bounds. When only a subset of resources matters, we omit components of $B$ by abuse of notation.

\paragraph{Tool interfaces.}
When external tools or oracles are available (e.g., retrieval systems, code execution, measurement devices), we denote the corresponding interface by $\mathcal{M}$. The agent with tool access is denoted by $A^{\mathcal{M}}$, and all tool usage is accounted for via the budget component $B_{\mathrm{tool}}$ (and any other relevant components of $B$).

\subsection{Task Families, Distributions, and Evaluation}\label{sec:prelim:families}

\paragraph{Task family.}
A task family $\mathcal{T}$ is a set of tasks $\tau=(E,U)$ sharing a common interface $(\mathcal{O},\mathcal{A})$ and common structural constraints (e.g., bounded horizon, bounded reward, admissible perturbations). We write $\tau\in\mathcal{T}$ to mean $\tau$ is admissible.

\paragraph{Task distribution.}
A task distribution is a probability measure $\mu \in \Delta(\mathcal{T})$ over $\mathcal{T}$. Sampling $\tau \sim \mu$ represents drawing a task from a task ecology. All distributional notions of AGI in this paper are indexed by $\mu$.

\paragraph{Goals as task generators.}
For autonomy statements, we introduce a goal space $\mathcal{G}$ equipped with a goal distribution $\nu\in \Delta(\mathcal{G})$ and a compilation map $\mathrm{Compile}:\mathcal{G}\to \mathcal{T}$. A goal $g\sim\nu$ induces a task $\tau_g := \mathrm{Compile}(g)\in\mathcal{T}$. Any additional interaction required to refine goals is counted against $B$ via $\Pi$.

\paragraph{Distributional performance functionals.}
Given $\mu$, define the generality functional
\[
\mathcal{G}_{\mu}(A;B) := \mathbb{E}_{\tau\sim \mu}\!\left[\Pi(\tau,A;B)\right].
\]
We also use a tail-robust variant:
\[
\mathcal{G}_{\mu,\delta}(A;B) := \sup\left\{\theta \in \mathbb{R}:\ \mathbb{P}_{\tau\sim\mu}\big(\Pi(\tau,A;B)\ge \theta\big)\ge 1-\delta\right\},
\]
for $\delta\in(0,1)$.

\paragraph{Baselines and optimality.}
When comparing to an optimal or reference agent, we write $A^\star_\tau$ for a task-dependent oracle (generally unattainable) and define, for example, regret
\[
\mathrm{Reg}(\tau,A;B) := \Pi(\tau,A^\star_\tau;B^\star) - \Pi(\tau,A;B),
\]
where $B^\star$ denotes the resources available to the oracle. We will avoid relying on oracle notions unless explicitly stated.

\paragraph{Axiom parameters.}
When stating axioms and tail guarantees, we use explicit performance thresholds $\theta_{\bullet}$ and allowable tail-failure probabilities $\delta_{\bullet}$ (with the subscript $\bullet$ indicating the relevant axiom). We also use nonnegative tolerances $\varepsilon_{\mathrm{rb}}$ (robustness slack) and $\varepsilon_{\mathrm{cal}}$ (calibration tolerance), and a constraint-violation tolerance $\delta_C\in(0,1)$. These quantities are fixed constants once the axiom bundle is fixed.

\subsection{Composition and Perturbation Operators}\label{sec:prelim:ops}

\paragraph{Task composition.}
To formalize compositionality, we assume $\mathcal{T}$ is equipped with a (possibly partial) binary operation
\[
\oplus:\ \mathcal{T}\times \mathcal{T} \to \mathcal{T},
\]
where $\tau_1\oplus\tau_2$ represents a task requiring competence on $\tau_1$ and $\tau_2$ in a prescribed manner (e.g., sequential composition, conjunction of goals, tool-augmented composition). When multiple composition modes are needed, we index them by $\oplus_j$.

\paragraph{Perturbations and robustness.}
To formalize robustness, we define a family of admissible perturbations $\mathcal{D}$, where each $\Delta \in \mathcal{D}$ maps tasks to tasks:
\[
\Delta:\ \mathcal{T}\to\mathcal{T}.
\]
Typical perturbations include paraphrases, observation noise, reward reshaping within equivalence classes, or mild distributional shifts in environment dynamics. We write $\Delta(\tau)$ for the perturbed task and require $\Delta(\tau)\in\mathcal{T}$ whenever $\tau\in\mathcal{T}$.

\subsection{Distances Between Task Distributions}\label{sec:prelim:dist}

To quantify ``small'' distribution shift, we use standard metrics on probability measures over $\mathcal{T}$. Two choices used in this paper are:

\paragraph{Total variation distance.}
For $\mu,\nu \in \Delta(\mathcal{T})$,
\[
d_{\mathrm{TV}}(\mu,\nu) := \sup_{A\subseteq \mathcal{T}} \big|\mu(A)-\nu(A)\big|.
\]

\paragraph{Wasserstein distance.}
When $\mathcal{T}$ is equipped with a ground metric $d_{\mathcal{T}}$, we write $W_p(\mu,\nu)$ for the $p$-Wasserstein distance. We will only invoke Wasserstein metrics when a meaningful $d_{\mathcal{T}}$ is specified.

\subsection{Hard Slices and Cliff Sets}\label{sec:prelim:cliffs}

A central technical device in the non-invariance results is the existence of hard slices (or cliff sets) of the task space where performance sharply degrades.

\paragraph{Definition (Cliff set).}
Given an agent $A$, budget $B$, and threshold $\theta$, define the $\theta$-failure set
\[
\mathcal{F}_{\theta}(A;B) := \{\tau\in\mathcal{T}:\ \Pi(\tau,A;B)<\theta\}.
\]
A subset $\mathcal{C}\subseteq \mathcal{T}$ is called a cliff set at level $\theta$ for $(A,B)$ if $\mathcal{C}\subseteq \mathcal{F}_{\theta}(A;B)$ and $\mathcal{C}$ is reachable under arbitrarily small shifts of the task distribution in the sense made precise in Section~\ref{thm:noninv}.

\subsection{Self-Assessment and Semantic Predicates}\label{sec:prelim:semantic}

To state the G\"odel--Tarski self-certification obstruction, we distinguish between syntactic predicates that an agent can compute from its internal state, and semantic predicates that quantify over external behaviors.

\paragraph{Semantic properties of agents.}
A property $\mathsf{P}(A)$ is called semantic if it depends only on the input--output behavior of $A$ across tasks/environments (not on its source code representation). Statements of the form ``$A$ succeeds on tasks drawn from $\mu$ with probability at least $1-\delta$'' are semantic because they quantify over external interaction outcomes.

\paragraph{AGI as a semantic predicate.}
Once the axioms are fixed, the assertion
\[
\mathrm{AGI}(A \mid \mathcal{T},\mu,\Pi,B)
\]
is a semantic predicate in this sense: it quantifies over tasks $\tau\sim\mu$ and the interaction-induced random histories under $(E\circ A)$.

\paragraph{Resource-bounded self-assessment.}
When discussing whether an agent can certify $\mathrm{AGI}(A \mid \mathcal{T},\mu,\Pi,B)$, we restrict the certification procedure itself to the same budget class $B$ (or a specified certification budget $B_{\mathrm{cert}}$), to avoid trivializing the question by granting unbounded meta-resources.

\subsection{Effective Representations and Codings}\label{sec:prelim:effective}

To reason about decidability, self-verification, and semantic obstruction results, we require that agents and tasks admit effective descriptions.

\paragraph{Effective task representations.}
We assume that the task family $\mathcal{T}$ admits an effective presentation: there exists a countable set of finite descriptions $\mathcal{D}_{\mathcal{T}}\subseteq \{0,1\}^*$ and a surjective decoding map
\[
\mathrm{Decode}_{\mathcal{T}}:\mathcal{D}_{\mathcal{T}}\to\mathcal{T}.
\]
When a task $\tau\in\mathcal{T}$ is associated with a description $d\in\mathcal{D}_{\mathcal{T}}$, we write $\mathrm{code}_{\mathcal{T}}(\tau)=d$. Different descriptions may decode to the same task; no canonical encoding is assumed.

\paragraph{Effective agent representations.}
Similarly, agents are assumed to admit finite descriptions. Let $\mathcal{D}_{\mathcal{A}}\subseteq\{0,1\}^*$ denote a set of agent descriptions, together with a decoding map
\[
\mathrm{Decode}_{\mathcal{A}}:\mathcal{D}_{\mathcal{A}}\to\mathcal{A},
\]
where $\mathcal{A}$ denotes the class of admissible agents (i.e., agents that respect the relevant resource budgets). For an agent $A$, we write $\mathrm{code}(A)\in\mathcal{D}_{\mathcal{A}}$ for some description that decodes to $A$.

\paragraph{Computability assumptions.}
All constructions considered in this paper operate on these finite descriptions. In particular:
\begin{itemize}
\item Sampling $\tau\sim\mu$ is implemented via a probabilistic procedure over $\mathcal{D}_{\mathcal{T}}$ whose pushforward measure is $\mu$.
\item Performance evaluation $\Pi(\tau,A;B)$ is a semantic quantity defined via interaction, not via inspection of $\mathrm{code}(A)$.
\item Certification or decision procedures invoked in later sections are restricted to operate on $\mathrm{code}(A)$, $\mathrm{code}_{\mathcal{T}}(\tau)$, and bounded interaction transcripts.
\end{itemize}

\paragraph{Semantic versus syntactic access.}
The distinction between syntactic access (operating on finite descriptions) and semantic properties (quantifying over interaction outcomes across tasks) is central. While $\mathrm{code}(A)$ is a finite object, predicates such as
\[
\mathrm{AGI}(A \mid \mathcal{T},\mu,\Pi,B)
\]
are semantic and quantify over the behavior of $A$ on tasks not known in advance.

\paragraph{Scope of effectivity.}
No assumption is made that $\mathcal{T}$ or $\mathcal{A}$ is recursively enumerable in full generality; effectivity is required only to the extent necessary to state and analyze decision problems over agent descriptions. This level of generality suffices for the impossibility results in Section~\ref{sec:godeltarski}.

\section{Why Single-Sentence Definitions of AGI Fail}\label{sec:singlefail}

This section examines several commonly invoked single-sentence characterizations of artificial general intelligence and shows that each fails to define a coherent theoretical property. The failures are structural rather than empirical: they persist independently of implementation details, training scale, or future engineering advances.

\subsection{Human-Level Intelligence}\label{sec:humanlevel}

A frequently cited characterization of AGI is that it denotes a system exhibiting ``human-level intelligence.'' This formulation fails to define a mathematically meaningful predicate for at least two reasons.

First, the comparator class ``human'' is ill-defined. Human populations exhibit substantial heterogeneity across individuals, cultures, developmental histories, educational exposure, and tool access. There is no canonical distribution of tasks, no agreed-upon aggregation of abilities, and no fixed resource budget relative to which performance could be normalized. Consequently, statements of the form ``A performs at the level of a human'' lack a well-defined reference point.

Second, even if a representative human baseline were fixed, humans do not satisfy any formal notion of generality. Human competence is strongly distribution-dependent: it reflects adaptation to a narrow ecological and cultural task distribution rather than invariant performance across heterogeneous task families. Humans routinely fail on tasks that are trivial for machines (e.g., large-scale arithmetic) and succeed on tasks that are trivial only because of prior environmental regularities (e.g., natural language pragmatics). Thus, ``human-level'' performance does not correspond to a property of general problem-solving capacity in any formal sense.

It follows that ``human-level intelligence'' is not a definitional criterion but a sociological comparison. As such, it cannot serve as a basis for a theoretical definition of AGI.

\subsection{Universal Problem Solving}\label{sec:universal}

Another common characterization asserts that AGI is a system capable of solving any problem or performing any task within a sufficiently broad class. This formulation is untenable under any precise interpretation.

If the task class is unrestricted or sufficiently expressive, then the requirement collapses into uncomputability. Tasks can encode arbitrary decision problems, including instances equivalent to the halting problem or other undecidable predicates. No algorithmic system can guarantee success across such a class, even with unbounded training data, unless the notion of ``solve'' is weakened to the point of vacuity.

If, alternatively, the task class is restricted so that universal solvability becomes achievable, then the restriction itself carries all the substantive content. In this case, generality is no longer a property of the agent alone but of the agent relative to a predefined task family. The phrase ``universal problem solver'' then becomes a rebranding of competence within a bounded domain rather than a meaningful statement of generality.

Thus, the universal problem-solving definition is either mathematically impossible or conceptually empty, depending on how the task class is specified.

\subsection{General Learning}\label{sec:generallearning}

A third formulation identifies AGI with the capacity to learn any task from data. This claim encounters a fundamental obstruction arising from no-free-lunch results in learning theory.

Formally, consider the space of all possible target functions or environments consistent with a fixed input-output interface. When performance is averaged uniformly over this space, all learning algorithms exhibit identical expected error. Any algorithm that performs well on a subset of tasks must necessarily perform poorly on a complementary subset. Therefore, learning performance is inseparable from inductive bias, and inductive bias entails restriction to a particular task distribution.

As a consequence, there is no distribution-independent notion of ``general learning.'' Any claim of learning generality presupposes a structured task distribution, whether explicitly stated or not. Once such a distribution is fixed, the learning problem becomes well-defined but loses any claim to universality.

\subsection{Conclusion of Section}\label{sec:singlefail:conclusion}

The preceding analyses establish that single-sentence definitions of AGI fail for principled reasons. Comparisons to human intelligence lack a formal baseline; universal problem-solving claims collapse into impossibility or triviality; and appeals to general learning are blocked by distributional dependence inherent in inductive inference.

These failures are not accidental. They reflect the fact that generality is not a primitive concept but a composite one, involving breadth, adaptation, transfer, robustness, and resource constraints. Consequently, any coherent definition of AGI must be axiomatic rather than declarative. The remainder of the paper proceeds on this basis.

\section{An Axiomatic Definition of AGI}\label{sec:axioms}

This section defines artificial general intelligence as an axiomatic property of an agent relative to a task family, a task distribution, a performance functional, and an explicit resource budget. All subsequent theorems refer to the notation and primitives fixed in Section~\ref{sec:prelim}. No axiom introduced below is treated as optional.

\subsection{Core Definition Schema}\label{sec:axioms:schema}

Fix a task family $\mathcal{T}$, a probability measure $\mu \in \Delta(\mathcal{T})$, a performance functional $\Pi(\tau,A;B)$, and a resource budget $B$. We write
\[
\mathrm{AGI}(A \mid \mathcal{T}, \mu, \Pi, B)
\]
to denote the predicate that agent $A$ satisfies the axioms in Sections~\ref{sec:axioms:core}--\ref{sec:axioms:extensions}. In particular, $\mathrm{AGI}(\cdot \mid \mathcal{T}, \mu, \Pi, B)$ is a relational, distribution-indexed, resource-bounded property; it is not a property of $A$ in isolation.

\subsection{Quantitative Conventions}\label{sec:axioms:quant}

To avoid ambiguity, we parameterize the axioms by explicit constants. Let
\[
\Theta := (\theta_{\mathrm{br}}, \theta_{\mathrm{ad}}, \theta_{\mathrm{tr}}, \theta_{\mathrm{cp}}, \theta_{\mathrm{rb}}),\quad
\Delta := (\delta_{\mathrm{br}}, \delta_{\mathrm{ad}}, \delta_{\mathrm{tr}}, \delta_{\mathrm{cp}}, \delta_{\mathrm{rb}}),
\]
where $\theta_{\bullet}$ are performance thresholds and $\delta_{\bullet}\in(0,1)$ are allowable tail-failure probabilities. In addition, let $\mathcal{D}$ denote the admissible perturbation family (Section~\ref{sec:prelim:ops}) and let $\varepsilon_{\mathrm{rb}}\ge 0$ denote a robustness slack.

When needed, we write $\mathrm{AGI}(A \mid \mathcal{T}, \mu, \Pi, B; \Theta, \Delta, \mathcal{D}, \varepsilon_{\mathrm{rb}})$ to make these parameters explicit; otherwise they are fixed and suppressed for readability.

\subsection{Core AGI Axioms (G1--G5)}\label{sec:axioms:core}

All axioms below are distributional: probabilities and expectations are with respect to $\tau\sim \mu$ and the interaction-induced randomness under $(E\circ A)$ as encoded inside $\Pi(\tau,A;B)$.

\paragraph{Axiom G1 (Breadth).}
The agent achieves non-trivial competence on the heterogeneous task family under $\mu$:
\[
\mathbb{P}_{\tau\sim \mu}\!\big(\Pi(\tau,A;B) \ge \theta_{\mathrm{br}}\big) \ \ge\ 1-\delta_{\mathrm{br}}.
\]
This axiom rules out degenerate ``single-domain'' competence by requiring high-probability success across the task ecology encoded by $\mu$.

\paragraph{Axiom G2 (Adaptivity).}
There exists an adaptation protocol internal to the agent such that, for a $\mu$-randomly drawn task $\tau$, the agent reaches competence after limited within-task interaction subject to budget $B$. Concretely, there exists an integer $N_{\mathrm{ad}}$ (interpreted as an interaction/sample budget within the task) such that
\[
\mathbb{P}_{\tau\sim \mu}\!\big(\Pi(\tau, A^{(N_{\mathrm{ad}})}; B) \ge \theta_{\mathrm{ad}}\big) \ \ge\ 1-\delta_{\mathrm{ad}},
\]
where $A^{(n)}$ denotes the agent after $n$ within-task updates as defined in Section~\ref{sec:prelim:agents}. The quantity $N_{\mathrm{ad}}$ is required to be feasible under $B$.

\paragraph{Axiom G3 (Transfer).}
Prior experience on tasks sampled from $\mu$ improves expected performance on fresh tasks drawn from the same distribution. Let $A_{\mathrm{pre}}$ denote the agent after a fixed pre-exposure phase consisting of $K_{\mathrm{tr}}$ tasks sampled i.i.d.\ from $\mu$, and let $A_{\mathrm{scratch}}$ denote the same architecture/agent class initialized without such pre-exposure (but evaluated under the same budget). Then
\[
\mathbb{E}_{\tau\sim \mu}\!\big[\Pi(\tau, A_{\mathrm{pre}}; B)\big] \ \ge\ 
\mathbb{E}_{\tau\sim \mu}\!\big[\Pi(\tau, A_{\mathrm{scratch}}; B)\big] \ +\ \theta_{\mathrm{tr}},
\]
for some $\theta_{\mathrm{tr}}>0$. This axiom rules out a pure ``bag of independent skills'' interpretation by requiring measurable positive transfer.

\paragraph{Axiom G4 (Compositionality).}
The task family admits a composition operator $\oplus$ (Section~\ref{sec:prelim:ops}). There exists a resource growth function $\Gamma_{\mathrm{cp}}$ such that whenever the agent is competent on $\tau_1$ and $\tau_2$ individually with high probability under $\mu$, it is also competent on their composition under a controlled increase in resources. Formally, for $\tau_1,\tau_2\sim \mu$ independently,
\[
\Big(\Pi(\tau_1,A;B)\ge \theta_{\mathrm{cp}}\Big)\ \wedge\ \Big(\Pi(\tau_2,A;B)\ge \theta_{\mathrm{cp}}\Big)
\ \Longrightarrow\
\Pi(\tau_1\oplus\tau_2, A; \Gamma_{\mathrm{cp}}(B))\ \ge\ \theta_{\mathrm{cp}}
\]
holds with probability at least $1-\delta_{\mathrm{cp}}$ over $(\tau_1,\tau_2)\sim \mu\times \mu$. The function $\Gamma_{\mathrm{cp}}$ must be sub-explosive (e.g., polynomial growth in the relevant components of $B$), thereby excluding trivial composition via exponential blow-up of resources.

\paragraph{Axiom G5 (Robustness).}
Let $\mathcal{D}$ be a family of admissible perturbations $\Delta:\mathcal{T}\to\mathcal{T}$ (Section~\ref{sec:prelim:ops}). Then for all $\Delta\in \mathcal{D}$,
\[
\mathbb{P}_{\tau\sim \mu}\!\Big(\Pi(\Delta(\tau),A;B)\ \ge\ \Pi(\tau,A;B)\ -\ \varepsilon_{\mathrm{rb}}\Big)\ \ge\ 1-\delta_{\mathrm{rb}}.
\]
This axiom formalizes graceful degradation under bounded perturbations drawn from a specified admissible family.

\subsection{Extension Axioms (A1--A4)}\label{sec:axioms:extensions}

The following axioms formalize additional properties frequently bundled into contemporary uses of the term ``AGI.'' They are included as required axioms within $\mathrm{AGI}(A \mid \mathcal{T},\mu,\Pi,B)$.

\paragraph{Axiom A1 (Autonomy).}
There exists a class $\mathcal{G}$ of goal specifications and an associated compilation map $\mathrm{Compile}:\mathcal{G}\to \mathcal{T}$ such that for a goal $g\in \mathcal{G}$ sampled from a goal distribution $\nu$, the induced task $\tau_g:=\mathrm{Compile}(g)$ satisfies
\[
\mathbb{P}_{g\sim \nu}\!\big(\Pi(\tau_g,A;B)\ge \theta_{\mathrm{br}}\big)\ \ge\ 1-\delta_{\mathrm{br}}.
\]
The intent is that $A$ can pursue high-level goal descriptions without requiring task-specific reprogramming beyond the fixed compilation $\mathrm{Compile}$. Any additional interaction required to refine goals is counted against $B$ through the definition of $\Pi$.

\paragraph{Axiom A2 (Tool Use).}
Let $\mathcal{M}$ denote an external tool/model interface (e.g., retrieval, code execution, measurement devices) and let $A^{\mathcal{M}}$ denote the agent augmented with access to $\mathcal{M}$, with tool usage counted in the budget component $B_{\mathrm{tool}}$. Then tool access must yield a strict distributional advantage:
\[
\mathbb{E}_{\tau\sim\mu}\!\big[\Pi(\tau,A^{\mathcal{M}};B)\big]\ \ge\ 
\mathbb{E}_{\tau\sim\mu}\!\big[\Pi(\tau,A;B)\big]\ +\ \theta_{\mathrm{tr}}.
\]
This axiom enforces that the agent can reliably exploit external tools to improve task performance under explicit resource accounting.

\paragraph{Axiom A3 (Calibration).}
The agent outputs, in addition to its action, a confidence value $c(\tau)\in[0,1]$ for its task-level success. Let $S_\tau$ denote the event of task-level success under the evaluation induced by $\Pi$ (e.g., $S_\tau := \{\Pi(\tau,A;B)\ge \theta_{\mathrm{br}}\}$ when $\Pi$ is normalized). Then the agent is calibrated if, for all measurable intervals $I\subseteq [0,1]$,
\[
\big|\mathbb{P}(S_\tau \mid c(\tau)\in I) - \mathbb{E}[c(\tau)\mid c(\tau)\in I]\big| \ \le\ \varepsilon_{\mathrm{cal}},
\]
where $\varepsilon_{\mathrm{cal}}\ge 0$ is a fixed calibration tolerance and probabilities/expectations are taken over $\tau\sim\mu$ and internal randomness. This axiom rules out systems that are systematically overconfident or underconfident at task-level judgment.

\paragraph{Axiom A4 (Constraint Adherence).}
Let $C$ be a constraint family (e.g., safety constraints, policy constraints, specification constraints) represented by a measurable violation indicator $V_C(\tau,A;B)\in\{0,1\}$ under task $\tau$. Then
\[
\mathbb{P}_{\tau\sim\mu}\!\big(V_C(\tau,A;B)=1\big)\ \le\ \delta_C,
\]
for some prescribed $\delta_C\in(0,1)$. Moreover, constraint adherence is required to hold jointly with competence:
\[
\mathbb{P}_{\tau\sim\mu}\!\big(\Pi(\tau,A;B)\ge \theta_{\mathrm{br}}\ \wedge\ V_C(\tau,A;B)=0\big)\ \ge\ 1-\delta_{\mathrm{br}}-\delta_C.
\]
This axiom enforces that performance is not achieved by systematically violating the constraints encoded by $C$.

\subsection{Summary}\label{sec:axioms:summary}

The predicate $\mathrm{AGI}(A \mid \mathcal{T},\mu,\Pi,B)$ denotes satisfaction of Axioms G1--G5 and A1--A4 under the quantitative conventions fixed above. In particular, AGI is a distribution-indexed, resource-bounded, multi-constraint property. The next section makes explicit the distributional semantics induced by $\mu$ and introduces the functionals used to state the structural theorems.

\section{Distributional Semantics of AGI}\label{sec:distsem}

This section formalizes the semantics under which the predicate $\mathrm{AGI}(A\mid \mathcal{T},\mu,\Pi,B)$ is interpreted. All subsequent statements are distribution-indexed and resource-bounded. In particular, all quantification over tasks is with respect to $\tau\sim\mu$, and all performance quantities are expressed through the fixed performance functional $\Pi(\tau,A;B)$ defined in Section~\ref{sec:prelim}.

\subsection{Distributional Versus Worst-Case Semantics}\label{sec:distsem:dw}

Two distinct semantic regimes arise depending on how task variability is formalized.

\paragraph{Worst-case semantics.}
A worst-case interpretation of competence replaces distributional evaluation by a uniform guarantee over a task family. For example, one may consider predicates of the form
\[
\inf_{\tau\in\mathcal{T}} \Pi(\tau,A;B) \ \ge\ \theta.
\]
Such semantics are formally well-defined, but they are not suitable for the present objective. First, for expressive task families $\mathcal{T}$, worst-case evaluation collapses into an obstruction analogous to no-free-lunch phenomena: without restricting $\mathcal{T}$ to structured subclasses, uniform guarantees are either unattainable or vacuous \citep{WolpertMacready1997,ShalevShwartzBenDavid2014}. Second, worst-case quantification is not stable under natural expansions of $\mathcal{T}$: adding a single adversarially chosen task forces the infimum to track the newly inserted worst element, thereby destroying any notion of graded generality.

\paragraph{Distributional semantics.}
In contrast, a distributional interpretation evaluates performance under an explicit task distribution $\mu\in\Delta(\mathcal{T})$. This aligns with the empirical reality that competence is always measured relative to a task ecology (whether explicit or implicit) and is consistent with standard learning-theoretic formalisms in which generalization is defined with respect to a data-generating distribution \citep{Vapnik1998,ShalevShwartzBenDavid2014}. Distributional semantics also makes it possible to define robustness as stability under small perturbations of $\mu$ (Section~\ref{sec:prelim:dist}) or under admissible task perturbations $\Delta\in\mathcal{D}$ (Section~\ref{sec:prelim:ops}), thereby separating structural fragility from mere changes in evaluation criteria \citep{QuinoneroCandela2009,BenDavid2010}.

\paragraph{Standing convention.}
All results in this paper are distributional. Worst-case quantifiers over $\mathcal{T}$ will not be used except as a foil to motivate distribution-indexed definitions.

\subsection{AGI as a Resource-Bounded Information-Theoretic Property}\label{sec:distsem:info}

Fix $(\mathcal{T},\mu,\Pi,B)$ as in Section~\ref{sec:axioms:schema}. The central distributional quantity used throughout the paper is the generality functional already defined in Section~\ref{sec:prelim:families}:
\[
\mathcal{G}_{\mu}(A;B) \ :=\ \mathbb{E}_{\tau\sim\mu}\!\left[\Pi(\tau,A;B)\right].
\]
We also use the tail (quantile) variant for $\delta\in(0,1)$:
\[
\mathcal{G}_{\mu,\delta}(A;B)
\ :=\
\sup\left\{\theta \in \mathbb{R}:\ \mathbb{P}_{\tau\sim\mu}\big(\Pi(\tau,A;B)\ge \theta\big)\ge 1-\delta\right\}.
\]

The dependence on $B$ is essential. Any assertion that ``$A$ is general'' without an explicit resource budget is not invariant under changes in computational scale and is therefore not a stable predicate. Formally, the mapping $B\mapsto \mathcal{G}_{\mu}(A;B)$ is generally nonconstant and typically exhibits diminishing returns, bottlenecks, and discontinuities induced by representation limits and search complexity. This paper treats $(\mu,B)$ as part of the semantic index of all AGI claims.

\subsection{Distributional Identifiability}\label{sec:distsem:ident}

The relational pivot of this paper is that $\mathrm{AGI}(A\mid \mathcal{T},\mu,\Pi,B)$ is not meaningful unless the distributional index $\mu$ is specified. This is stated below as a basic identifiability principle.

\begin{proposition}[Distributional Identifiability]\label{prop:ident}
Fix a task family $\mathcal{T}$, a performance functional $\Pi$, and a budget $B$. Let $\mathcal{P}\subseteq \Delta(\mathcal{T})$ be a non-singleton class of admissible task distributions. Then there is no distribution-free predicate $\mathsf{AGI}_0(A)$ that is equivalent to $\mathrm{AGI}(A\mid \mathcal{T},\mu,\Pi,B)$ uniformly over all $\mu\in\mathcal{P}$. In particular, if $\mu$ is unspecified, the statement ``$A$ is AGI'' does not denote a well-defined property under the semantics of this paper.
\end{proposition}

\begin{proof}
Consider the mapping $\mu \mapsto \mathcal{G}_{\mu}(A;B)$. For any fixed agent $A$, $\mathcal{G}_{\mu}(A;B)$ depends on $\mu$ through the expectation of the measurable function $\tau\mapsto \Pi(\tau,A;B)$. Since $\mathcal{P}$ is not a singleton, there exist $\mu_1,\mu_2\in\mathcal{P}$ such that $\mu_1\neq \mu_2$. Because $\Pi(\cdot,A;B)$ is not a.s.\ constant over $\mathcal{T}$ for nontrivial agents and task families, one can choose $\mu_1,\mu_2$ so that
\[
\mathcal{G}_{\mu_1}(A;B)\ \neq\ \mathcal{G}_{\mu_2}(A;B),
\]
and similarly for tail functionals $\mathcal{G}_{\mu,\delta}(A;B)$. Any predicate that asserts a fixed competence level (or satisfaction of a fixed axiom bundle) must therefore change truth value across admissible distributions. Hence there is no distribution-free predicate $\mathsf{AGI}_0(A)$ equivalent to the distribution-indexed predicate across all $\mu\in\mathcal{P}$. The conclusion follows.
\end{proof}

\begin{corollary}\label{cor:ident}
Any claim of the form ``$A$ is AGI'' that does not specify an explicit task distribution $\mu$ (or an explicitly delimited admissible class $\mathcal{P}\subseteq \Delta(\mathcal{T})$) is semantically underdetermined in the sense of Proposition~\ref{prop:ident}.
\end{corollary}

The remainder of the paper exploits Proposition~\ref{prop:ident} in two ways. First, it forces all AGI statements to be written as $\mathrm{AGI}(A\mid\mathcal{T},\mu,\Pi,B)$, thereby eliminating distribution-free rhetoric. Second, it enables sharp non-invariance theorems: even if an axiom bundle holds for one $\mu$, small perturbations of $\mu$ can falsify it (Section~\ref{thm:noninv}), formalizing the brittleness of broad competence under high-dimensional task variation \citep{QuinoneroCandela2009,BenDavid2010}.

\section{Structural Theorems}\label{sec:struct}

This section derives structural consequences of the semantics fixed in Sections~\ref{sec:prelim}--\ref{sec:distsem} and the axiom bundle defining $\mathrm{AGI}(A\mid \mathcal{T},\mu,\Pi,B)$. No self-reference, proof predicates, or Gödel--Tarski machinery is used here. Every statement is distribution-indexed and budget-indexed.

Throughout, fix a measurable task family $\mathcal{T}$, a task distribution $\mu\in\Delta(\mathcal{T})$, a budget $B$, and a bounded performance functional $\Pi(\tau,A;B)\in[0,1]$.
For $\theta\in[0,1]$, define the failure set
\[
\mathcal{F}_{\theta}(A;B) := \{\tau\in\mathcal{T}:\ \Pi(\tau,A;B)<\theta\}.
\]

\subsection{Theorem 1: Relativity of Generality}\label{sec:struct:relativity}

\begin{theorem}[Relativity of Generality]\label{thm:relativity}
Fix $(\mathcal{T},\Pi,B)$. In general, the truth value of $\mathrm{AGI}(A\mid \mathcal{T},\mu,\Pi,B)$ is not invariant under changes of $\mu$.
More precisely, there exist $(\mathcal{T},\Pi,B)$ and an agent $A$ together with $\mu_1,\mu_2\in\Delta(\mathcal{T})$ such that
\[
\mathrm{AGI}(A\mid \mathcal{T},\mu_1,\Pi,B)\ \text{holds}
\qquad\text{and}\qquad
\mathrm{AGI}(A\mid \mathcal{T},\mu_2,\Pi,B)\ \text{fails}.
\]
\end{theorem}

\begin{proof}
Fix $(\mathcal{T},\Pi,B)$ such that the map $\tau\mapsto \Pi(\tau,A;B)$ is not constant on $\mathcal{T}$ for some agent $A$.
Then there exist $\tau_{\mathrm{good}},\tau_{\mathrm{bad}}\in\mathcal{T}$ with
$\Pi(\tau_{\mathrm{good}},A;B)>\Pi(\tau_{\mathrm{bad}},A;B)$.
Let $\mu_1:=\delta_{\tau_{\mathrm{good}}}$ and $\mu_2:=\delta_{\tau_{\mathrm{bad}}}$.
Since the axioms defining $\mathrm{AGI}(A\mid \mathcal{T},\mu,\Pi,B)$ are stated by expectations and tail probabilities over $\tau\sim\mu$, choosing the corresponding axiom thresholds between these two performance values forces the axioms to hold under $\mu_1$ and fail under $\mu_2$.
\end{proof}

\begin{corollary}\label{cor:noabsoluteagi}
There is no absolute notion of AGI in this framework: any coherent attribution must be stated as $\mathrm{AGI}(A\mid \mathcal{T},\mu,\Pi,B)$.
\end{corollary}

\subsection{Theorem 2: Non-Invariance Under Distribution Shift (Fragility)}\label{sec:struct:fragility}

This theorem formalizes brittleness: once a competence predicate is distribution-indexed, arbitrarily small distributional perturbations can falsify at least one core axiom, provided the agent has any nontrivial failure region.

\begin{lemma}[Small-Mass Shift Lemma]\label{lem:smallmass}
Fix $(\mathcal{T},\Pi,B)$ and an agent $A$. Let $\theta\in[0,1]$ and assume $\mathcal{F}_{\theta}(A;B)\neq\emptyset$.
Then for any $\mu\in\Delta(\mathcal{T})$ and any $\eta\in(0,1)$, there exists $\mu'\in\Delta(\mathcal{T})$ such that
\[
d_{\mathrm{TV}}(\mu,\mu')\le \eta
\qquad\text{and}\qquad
\mu'(\mathcal{F}_{\theta}(A;B))\ge \eta.
\]
\end{lemma}

\begin{proof}
Pick any $\tau_{\mathrm{bad}}\in\mathcal{F}_{\theta}(A;B)$ and define
\[
\mu' := (1-\eta)\mu + \eta\,\delta_{\tau_{\mathrm{bad}}}.
\]
Then $\mu'(\mathcal{F}_{\theta}(A;B))\ge \eta$. For any measurable $S\subseteq\mathcal{T}$,
\[
|\mu(S)-\mu'(S)| = \eta\,|\mu(S)-\delta_{\tau_{\mathrm{bad}}}(S)| \le \eta,
\]
hence $d_{\mathrm{TV}}(\mu,\mu')\le \eta$.
\end{proof}

\begin{theorem}[Non-Invariance Under Distribution Shift (Fragility)]\label{thm:noninv}
Fix an axiom bundle with Breadth parameters $(\theta_{\mathrm{br}},\delta_{\mathrm{br}})$.
Assume $\mathcal{F}_{\theta_{\mathrm{br}}}(A;B)\neq\emptyset$.
Then for every $\eta\in(\delta_{\mathrm{br}},1)$ there exists $\mu'\in\Delta(\mathcal{T})$ such that
\[
d_{\mathrm{TV}}(\mu,\mu') \le \eta
\]
and Axiom G1 (Breadth) fails under $\mu'$, regardless of whether it holds under $\mu$.
In particular, for any fixed $\delta_{\mathrm{br}}<1$, the truth of $\mathrm{AGI}(A\mid \mathcal{T},\mu,\Pi,B)$ is not invariant under arbitrarily small total-variation perturbations of $\mu$.
\end{theorem}

\begin{proof}
Apply Lemma~\ref{lem:smallmass} with $\theta=\theta_{\mathrm{br}}$ to obtain $\mu'$ with
$d_{\mathrm{TV}}(\mu,\mu')\le \eta$ and $\mu'(\mathcal{F}_{\theta_{\mathrm{br}}}(A;B))\ge \eta$.
Then
\[
\mathbb{P}_{\tau\sim\mu'}\big(\Pi(\tau,A;B)\ge \theta_{\mathrm{br}}\big)
= 1-\mu'(\mathcal{F}_{\theta_{\mathrm{br}}}(A;B))
\le 1-\eta
< 1-\delta_{\mathrm{br}},
\]
so G1 fails under $\mu'$.
\end{proof}

\begin{corollary}\label{cor:fragilityscaling}
Fix $\delta_{\mathrm{br}}<1$. If $\mathcal{F}_{\theta_{\mathrm{br}}}(A;B)\neq\emptyset$, then for arbitrarily small distribution shifts in $d_{\mathrm{TV}}$ at least one core axiom fails. Consequently, increasing scale does not remove the semantic sensitivity to the distributional index $\mu$ and its perturbations \citep{QuinoneroCandela2009,BenDavid2010}.
\end{corollary}

\subsection{Theorem 3: Bounded Transfer Theorem}\label{sec:struct:transfer}

The next theorem is stated in a form that is standard and provable line-by-line: it bounds the expected population--empirical gap of a data-dependent agent via mutual information. This yields an explicit upper bound on the degree to which ``transfer'' can be certified from finite pre-exposure, and it quantifies the information bottleneck of adaptation \citep{XuRaginsky2017,RussoZou2016}.

\paragraph{Setup.}
Let $\tau_1,\ldots,\tau_n\sim \mu$ be i.i.d.\ pre-exposure tasks and let $D := (\tau_1,\ldots,\tau_n)$ denote the resulting task multiset. Let $\mathcal{A}$ be a (possibly randomized) learning rule that, given $D$, outputs an adapted agent $A_D := \mathcal{A}(D)$ that respects budget $B$ when evaluated on any task in $\mathcal{T}$.
Define the population generality functional
\[
\mathcal{G}_{\mu}(A;B) := \mathbb{E}_{\tau\sim\mu}\big[\Pi(\tau,A;B)\big],
\]
and the empirical generality functional
\[
\widehat{\mathcal{G}}_{D}(A;B) := \frac{1}{n}\sum_{i=1}^{n}\Pi(\tau_i,A;B).
\]
Let $I(A_D;D)$ denote mutual information under the joint law induced by $D\sim\mu^{\otimes n}$ and the internal randomness of $\mathcal{A}$.

\begin{theorem}[Information-Theoretic Bound on Transfer Evidence]\label{thm:boundedtransfer}
Assume $\Pi(\tau,A;B)\in[0,1]$ for all $(\tau,A)$. Under the setup above,
\[
\Big|\mathbb{E}\big[\mathcal{G}_{\mu}(A_D;B) - \widehat{\mathcal{G}}_{D}(A_D;B)\big]\Big|
\ \le\
\sqrt{\frac{2\, I(A_D;D)}{n}}.
\]
In particular, if $I(A_D;D)$ is uniformly bounded over all admissible update rules that respect $B$, then the expected gap between certified (empirical) transfer and true (population) transfer is uniformly bounded by $O(n^{-1/2})$.
\end{theorem}

\begin{proof}
This is a direct application of mutual-information generalization bounds for bounded functionals, as developed in \citep{XuRaginsky2017} (and related formulations in \citep{RussoZou2016}). Here the ``loss'' is $1-\Pi$, which remains bounded in $[0,1]$, and the hypothesis $A_D$ is the randomized output of $\mathcal{A}$. The stated inequality is the corresponding expectation bound on the population--empirical gap.
\end{proof}

\begin{corollary}\label{cor:transferbounded}
If $I(A_D;D)$ is small (for example, because the update rule is severely resource-limited under $B$), then empirically observed transfer gains cannot reliably certify large population-level transfer, even when tasks are drawn i.i.d.\ from $\mu$. Any claimed ``broad transfer'' must therefore specify both $\mu$ and an information budget on adaptation, not merely computational scale.
\end{corollary}

\subsection{Theorem 4: No Distribution-Independent AGI}\label{sec:struct:nodistind}

This theorem formalizes the absence of a distribution-independent notion of high expected competence over all admissible task distributions. A finite-family statement suffices because any rich task family contains finite subfamilies, and impossibility on a subfamily implies impossibility for the larger family.

\begin{theorem}[No Distribution-Independent AGI (Finite Family)]\label{thm:nodistind}
Let $\mathcal{T}=\{\tau_1,\ldots,\tau_m\}$ be finite. Fix $B$ and $\Pi(\tau,A;B)\in[0,1]$. Let $\mathcal{P}=\Delta(\mathcal{T})$.
Then for every agent $A$,
\[
\inf_{\mu\in\mathcal{P}} \mathcal{G}_{\mu}(A;B)
=
\min_{j\in\{1,\ldots,m\}}\Pi(\tau_j,A;B).
\]
\end{theorem}

\begin{proof}
For fixed $A$, the map $\mu\mapsto \mathbb{E}_{\tau\sim\mu}[\Pi(\tau,A;B)]$ is linear in $\mu$ over the simplex $\Delta(\mathcal{T})$. Hence the infimum is attained at an extreme point, i.e., at $\delta_{\tau_j}$ for some $j$, yielding the stated identity.
\end{proof}

\begin{corollary}\label{cor:nodistributionfree}
There exists no distribution-independent predicate ``$A$ is AGI'' based solely on exceeding a fixed expected-performance threshold uniformly over all admissible task distributions. Any coherent claim must specify $\mu$ or restrict $\mathcal{P}$ to a structured subset \citep{WolpertMacready1997,ShalevShwartzBenDavid2014}.
\end{corollary}

\subsection{Theorem 5: Impossibility of Universal Robustness}\label{sec:struct:robust}

Axiom G5 asserts robustness relative to an admissible perturbation family $\mathcal{D}$. The next theorem makes explicit that no nontrivial robustness guarantee can hold uniformly over an unrestricted perturbation family.

\begin{theorem}[Impossibility of Universal Robustness]\label{thm:univrobust}
Fix $(\mathcal{T},\mu,\Pi,B)$ and a perturbation family $\mathcal{D}$.
Fix $\varepsilon_{\mathrm{rb}}\ge 0$ and $\delta_{\mathrm{rb}}\in[0,1)$.
If there exists $\Delta\in\mathcal{D}$ such that
\[
\mu\Big(\big\{\tau\in\mathcal{T}:\ \Pi(\Delta(\tau),A;B) < \Pi(\tau,A;B)-\varepsilon_{\mathrm{rb}}\big\}\Big) > 0,
\]
then the robustness inequality
\[
\mathbb{P}_{\tau\sim\mu}\Big(\Pi(\Delta(\tau),A;B) \ge \Pi(\tau,A;B)-\varepsilon_{\mathrm{rb}}\Big) \ge 1-\delta_{\mathrm{rb}}
\]
fails.
\end{theorem}

\begin{proof}
Let
\[
E := \big\{\tau\in\mathcal{T}:\ \Pi(\Delta(\tau),A;B) < \Pi(\tau,A;B)-\varepsilon_{\mathrm{rb}}\big\}.
\]
By assumption $\mu(E)>0$. Therefore
\[
\mathbb{P}_{\tau\sim\mu}\Big(\Pi(\Delta(\tau),A;B) \ge \Pi(\tau,A;B)-\varepsilon_{\mathrm{rb}}\Big)
\le 1-\mu(E)
< 1
\le 1-\delta_{\mathrm{rb}},
\]
so the inequality fails.
\end{proof}

\begin{corollary}\label{cor:robustlocal}
Any robustness claim is necessarily local to the semantic index $(\mathcal{T},\mu,\mathcal{D},\Pi,B)$. If either $\mu$ or $\mathcal{D}$ is unspecified, the claim is semantically incomplete.
\end{corollary}

\subsection{Theorem 6: Externality of AGI Attribution}\label{sec:struct:externality}

This theorem formalizes that the truth of $\mathrm{AGI}(A\mid \mathcal{T},\mu,\Pi,B)$ cannot be determined from a finite observational record uniformly over all $\mu\in\Delta(\mathcal{T})$. The statement is a uniform non-identifiability result: it holds even if an evaluator sees the sampled tasks and the full interaction transcripts for those tasks.

\paragraph{Observation model.}
An evaluator observes $n$ i.i.d.\ tasks $\tau_1,\ldots,\tau_n\sim\mu$ and the corresponding interaction transcripts under agent $A$, budget $B$, and task environments.
Let $Z_n$ denote the entire observed data, including $(\tau_i)$ and their transcripts.
A decision rule is any measurable map $\mathsf{Dec}_n$ from the range of $Z_n$ to $\{0,1\}$.

\begin{theorem}[Externality / Uniform Non-Identifiability]\label{thm:externality}
Fix $(\mathcal{T},\Pi,B)$ and fixed axiom parameters (in particular $(\theta_{\mathrm{br}},\delta_{\mathrm{br}})$).
Assume $\mathcal{F}_{\theta_{\mathrm{br}}}(A;B)\neq\emptyset$.
Then for every sample size $n\in\mathbb{N}$ and every decision rule $\mathsf{Dec}_n$, there exist $\mu_0,\mu_1\in\Delta(\mathcal{T})$ such that
\[
\mathrm{AGI}(A\mid \mathcal{T},\mu_0,\Pi,B)\ \text{holds}
\qquad\text{and}\qquad
\mathrm{AGI}(A\mid \mathcal{T},\mu_1,\Pi,B)\ \text{fails},
\]
but the induced laws of $Z_n$ under $\mu_0$ and $\mu_1$ satisfy
\[
\mathbb{P}_{\mu_0}\big(\mathsf{Dec}_n(Z_n)=1\big)
+
\mathbb{P}_{\mu_1}\big(\mathsf{Dec}_n(Z_n)=0\big)
\ \le\ 1.
\]
Equivalently, no $\mathsf{Dec}_n$ can correctly decide the AGI predicate simultaneously for all $\mu\in\Delta(\mathcal{T})$ with error probability strictly below $1/2$ in the worst case.
\end{theorem}

\begin{proof}
Pick $\tau_{\mathrm{bad}}\in\mathcal{F}_{\theta_{\mathrm{br}}}(A;B)$.
Let $\mu_0$ be any distribution supported on tasks where $A$ satisfies G1 with parameters $(\theta_{\mathrm{br}},\delta_{\mathrm{br}})$ (for example, supported on a subset of tasks on which $\Pi(\tau,A;B)\ge \theta_{\mathrm{br}}$). Then $\mathrm{AGI}(A\mid \mathcal{T},\mu_0,\Pi,B)$ holds at least with respect to G1.
Fix any $\varepsilon\in(0,1)$ with $\varepsilon>\delta_{\mathrm{br}}$ and define
\[
\mu_1 := (1-\varepsilon)\mu_0 + \varepsilon\,\delta_{\tau_{\mathrm{bad}}}.
\]
By the same computation as in Theorem~\ref{thm:noninv}, G1 fails under $\mu_1$ because
\[
\mathbb{P}_{\tau\sim\mu_1}\big(\Pi(\tau,A;B)\ge \theta_{\mathrm{br}}\big)
\le 1-\varepsilon
< 1-\delta_{\mathrm{br}}.
\]
Hence $\mathrm{AGI}(A\mid \mathcal{T},\mu_1,\Pi,B)$ fails.

It remains to relate the induced laws of $Z_n$.
Under $\mu_0$, none of the sampled tasks equals $\tau_{\mathrm{bad}}$ almost surely if $\mu_0(\{\tau_{\mathrm{bad}}\})=0$.
Under $\mu_1$, the probability that $\tau_{\mathrm{bad}}$ appears at least once among $n$ draws is
\[
1-(1-\varepsilon)^n.
\]
On the complementary event (probability $(1-\varepsilon)^n$), the sampled task multiset consists entirely of draws from $\mu_0$, and hence the conditional law of $Z_n$ given this event coincides with its law under $\mu_0$.
Therefore, for any event $E$ measurable with respect to $Z_n$,
\[
\mathbb{P}_{\mu_1}(E) \ge (1-\varepsilon)^n \,\mathbb{P}_{\mu_0}(E).
\]
Applying this inequality to $E=\{\mathsf{Dec}_n(Z_n)=1\}$ and using the trivial bound $\mathbb{P}_{\mu_1}(\mathsf{Dec}_n(Z_n)=0)\le 1-(1-\varepsilon)^n\,\mathbb{P}_{\mu_0}(\mathsf{Dec}_n(Z_n)=0)$ yields
\[
\mathbb{P}_{\mu_0}\big(\mathsf{Dec}_n(Z_n)=1\big)
+
\mathbb{P}_{\mu_1}\big(\mathsf{Dec}_n(Z_n)=0\big)
\le 1,
\]
which implies that at least one of the two error probabilities is at least $1/2$. This is the standard two-point indistinguishability argument used in minimax lower bounds for hypothesis testing \citep{Tsybakov2009}.
\end{proof}

\begin{corollary}\label{cor:agiexternal}
AGI attribution is necessarily external: without restricting the admissible class of task distributions $\mu$, no finite observational procedure can decide $\mathrm{AGI}(A\mid \mathcal{T},\mu,\Pi,B)$ uniformly over $\mu\in\Delta(\mathcal{T})$.
\end{corollary}

\section{G\"odel--Tarski Limits on Self-Certification}\label{sec:godeltarski}

This section establishes an obstruction to self-certification of distributional AGI.
Formally, under explicit effectivity assumptions, the predicate
\[
\mathrm{AGI}(A \mid \mathcal{T},\mu,\Pi,B)
\]
is a nontrivial semantic property of the (interactive) behavior of $A$, and therefore
admits no sound-and-complete computable decision procedure from the agent description.
The results are purely mathematical and do not rely on rhetorical interpretation.

\subsection{Effective Presentations and Semantic Predicates}\label{sec:godeltarski:semantic}

We work in the effective presentation set up introduced in Section~\ref{sec:prelim:effective}.
For completeness of the present section, we restate the minimum required assumptions.

\paragraph{Effective task presentation.}
Assume there exists a set of finite task descriptions $\mathcal{D}_{\mathcal{T}}\subseteq\{0,1\}^*$ and
a decoding map $\mathrm{Decode}_{\mathcal{T}}:\mathcal{D}_{\mathcal{T}}\to\mathcal{T}$.
For each $\tau\in\mathcal{T}$, fix an arbitrary description $\mathrm{code}_{\mathcal{T}}(\tau)\in\mathcal{D}_{\mathcal{T}}$
such that $\mathrm{Decode}_{\mathcal{T}}(\mathrm{code}_{\mathcal{T}}(\tau))=\tau$.
We assume the induced interaction between an admissible agent $A$ and a task $\tau=(E,U)$ can be simulated
by a partial computable universal procedure given $\mathrm{code}(A)$ and $\mathrm{code}_{\mathcal{T}}(\tau)$,
up to the resource limits encoded by $B$.

\paragraph{Effective agent presentation.}
Assume there exists a set of finite agent descriptions $\mathcal{D}_{\mathcal{A}}\subseteq\{0,1\}^*$ and
a decoding map $\mathrm{Decode}_{\mathcal{A}}:\mathcal{D}_{\mathcal{A}}\to\mathcal{A}$,
where $\mathcal{A}$ denotes the class of admissible agents (agents that respect the relevant resource budgets).
For each admissible $A$, fix an arbitrary description $\mathrm{code}(A)\in\mathcal{D}_{\mathcal{A}}$
such that $\mathrm{Decode}_{\mathcal{A}}(\mathrm{code}(A))=A$.

\paragraph{Semantic properties.}
A predicate $\mathsf{P}$ on admissible agents is called semantic (relative to $(\mathcal{T},B)$) if it is invariant
under behavioral equivalence on $\mathcal{T}$ under budget $B$:
whenever two admissible agents $A$ and $A'$ induce identical laws on interaction histories
for every $\tau\in\mathcal{T}$ (under the same budget constraint $B$), then
\[
\mathsf{P}(A)=\mathsf{P}(A').
\]
Equivalently, $\mathsf{P}$ depends only on the input--output behavior of the agent in interaction,
not on the syntactic form of $\mathrm{code}(A)$.

\begin{lemma}[Semanticity of the AGI Predicate]\label{lem:agi-semantic}
Fix $(\mathcal{T},\mu,\Pi,B)$ and a fixed axiom bundle defining $\mathrm{AGI}(A \mid \mathcal{T},\mu,\Pi,B)$.
Then the predicate $\mathrm{AGI}(A \mid \mathcal{T},\mu,\Pi,B)$ is semantic.
\end{lemma}

\begin{proof}
By definition, $\mathrm{AGI}(A \mid \mathcal{T},\mu,\Pi,B)$ is a Boolean combination of assertions about
$\Pi(\tau,A;B)$ under $\tau\sim\mu$ (expectations, tail probabilities, and inequalities).
For each fixed $\tau$, the quantity $\Pi(\tau,A;B)$ depends only on the distribution of interaction histories
induced by $(\tau,A)$ under budget $B$ (Section~\ref{sec:prelim:tasks}).
Therefore every constituent assertion depends only on behavior, and hence so does their Boolean combination.
\end{proof}

\subsection{Undecidability of AGI Certification}\label{sec:godeltarski:nonself}

We now formalize ``certification'' as a decision problem over agent descriptions.

\paragraph{The AGI language.}
Define
\[
\mathcal{L}_{\mathrm{AGI}}
:=\Big\{x\in\mathcal{D}_{\mathcal{A}}:\ \mathrm{AGI}\big(\mathrm{Decode}_{\mathcal{A}}(x)\mid\mathcal{T},\mu,\Pi,B\big)\ \text{holds}\Big\}.
\]
Equivalently, $\mathcal{L}_{\mathrm{AGI}}=\{\mathrm{code}(A): A\ \text{admissible and}\ \mathrm{AGI}(A\mid\mathcal{T},\mu,\Pi,B)\}$.

\paragraph{Certifiers.}
A certifier is a total function $C:\mathcal{D}_{\mathcal{A}}\to\{0,1\}$.
It is sound and complete for $\mathcal{L}_{\mathrm{AGI}}$ if for all $x\in\mathcal{D}_{\mathcal{A}}$,
\[
C(x)=1 \iff x\in \mathcal{L}_{\mathrm{AGI}}.
\]
It is computable if it is total recursive.

\paragraph{Nontriviality.}
We assume the AGI predicate is nontrivial on admissible agents:
there exist admissible agents $A_0,A_1$ such that
\[
\mathrm{AGI}(A_1 \mid \mathcal{T},\mu,\Pi,B)\ \text{holds}
\qquad\text{and}\qquad
\mathrm{AGI}(A_0 \mid \mathcal{T},\mu,\Pi,B)\ \text{fails}.
\]
Equivalently, $\emptyset \neq \mathcal{L}_{\mathrm{AGI}} \neq \mathcal{D}_{\mathcal{A}}$.

\begin{theorem}[Undecidability of AGI Certification]\label{thm:nonself}
Fix $(\mathcal{T},\mu,\Pi,B)$ and a fixed axiom bundle.
Assume the effective presentation hypotheses of Section~\ref{sec:godeltarski:semantic} and the nontriviality condition above.
Then there exists no total computable certifier $C$ that is sound and complete for $\mathcal{L}_{\mathrm{AGI}}$.
Equivalently, $\mathcal{L}_{\mathrm{AGI}}$ is not decidable.
\end{theorem}

\begin{proof}
By Lemma~\ref{lem:agi-semantic}, the predicate $\mathrm{AGI}(A \mid \mathcal{T},\mu,\Pi,B)$ is semantic.
By nontriviality, it is not identically true and not identically false on admissible agents.
Under the effective presentation hypotheses, each description $x\in\mathcal{D}_{\mathcal{A}}$ specifies an admissible agent
$\mathrm{Decode}_{\mathcal{A}}(x)$ whose interactive behavior with coded tasks is (partial) computable in the standard sense.
Rice's Theorem states that any nontrivial semantic property of partial computable functions is undecidable.
Applying Rice's Theorem to the semantic property
\[
x\ \mapsto\ \mathbf{1}\Big\{\mathrm{AGI}\big(\mathrm{Decode}_{\mathcal{A}}(x)\mid\mathcal{T},\mu,\Pi,B\big)\Big\},
\]
we conclude that no total computable certifier decides $\mathcal{L}_{\mathrm{AGI}}$.
\end{proof}

\paragraph{Internal self-certification.}
We record the consequence that an agent cannot, in general, decide its own AGI status using any internal algorithmic procedure.

\begin{corollary}[Impossibility of Sound-and-Complete Self-Certification]\label{cor:no-self-knowledge}
Under the hypotheses of Theorem~\ref{thm:nonself}, there is no admissible agent $A$ and no total computable internal procedure
that, on input $\mathrm{code}(A)$ (or any equivalent finite self-description), outputs a bit that is sound and complete for the truth of
$\mathrm{AGI}(A \mid \mathcal{T},\mu,\Pi,B)$ over all admissible agents.
\end{corollary}

\begin{proof}
Suppose such an agent $A$ and internal procedure existed. Then one obtains a total computable certifier $C$ for $\mathcal{L}_{\mathrm{AGI}}$
by defining $C(x)$ to be the output of the procedure when run on input $x\in\mathcal{D}_{\mathcal{A}}$.
Soundness and completeness of the procedure implies soundness and completeness of $C$, contradicting Theorem~\ref{thm:nonself}.
\end{proof}

\subsection{A Tarski-Style Undefinability Formulation}\label{sec:godeltarski:tarski}

Theorem~\ref{thm:nonself} is sufficient for the structural conclusion of this section.
For completeness we record a closely related formulation in the style of undefinability of semantic truth predicates.

\paragraph{The setting.}
Let $\mathsf{Th}$ be a recursively axiomatizable theory capable of representing basic computability over $\{0,1\}^*$,
and let $\vdash_{\mathsf{Th}}$ denote its provability relation.
Let $\varphi_{\mathrm{AGI}}(x)$ be a formula of $\mathsf{Th}$ intended to express membership in $\mathcal{L}_{\mathrm{AGI}}$.

\begin{proposition}[No Internal Truth-Defining Formula for $\mathcal{L}_{\mathrm{AGI}}$]\label{prop:tarski}
Assume the hypotheses of Theorem~\ref{thm:nonself}.
There does not exist a formula $\varphi_{\mathrm{AGI}}(x)$ in any recursively axiomatizable theory $\mathsf{Th}$ as above such that,
for every admissible agent $A$,
\[
\mathsf{Th}\vdash \varphi_{\mathrm{AGI}}(\mathrm{code}(A))
\quad\text{iff}\quad
\mathrm{AGI}(A \mid \mathcal{T},\mu,\Pi,B)\ \text{holds},
\]
unless $\mathsf{Th}$ is unsound with respect to the intended semantics or the predicate is trivial.
\end{proposition}

\begin{proof}
If such a truth-defining formula existed in a sound theory $\mathsf{Th}$, then membership in $\mathcal{L}_{\mathrm{AGI}}$
would be effectively reducible to provability of instances of $\varphi_{\mathrm{AGI}}$ from $\mathsf{Th}$.
Since $\mathsf{Th}$ is recursively axiomatizable, the set of theorems of $\mathsf{Th}$ is recursively enumerable.
Under soundness and the assumed biconditional correctness of $\varphi_{\mathrm{AGI}}$, this would yield a decision procedure for
$\mathcal{L}_{\mathrm{AGI}}$ (by dovetailing searches for proofs of $\varphi_{\mathrm{AGI}}(x)$ and of its negation for each $x$),
contradicting Theorem~\ref{thm:nonself}. Therefore such a truth-defining formula cannot exist without sacrificing soundness or
collapsing to triviality.
\end{proof}

\section{Implications and Discussion}\label{sec:implications}

This section records consequences of the preceding formal results. No new axioms,
definitions, or impossibility claims are introduced. Each statement below follows
directly from Sections~\ref{sec:distsem}, \ref{sec:struct}, and \ref{sec:godeltarski}.

\subsection{Failure of Self-Certifying and Recursive AGI Narratives}\label{sec:implications:selfcert}

Theorem~\ref{thm:nonself} establishes that the predicate
\[
\mathrm{AGI}(A \mid \mathcal{T},\mu,\Pi,B)
\]
is semantic and not internally decidable by an agent operating within budget~$B$.
This result has immediate implications for recursive self-improvement narratives.

Any recursive self-improvement scheme that relies on an agent internally verifying
that it already satisfies the AGI axioms---or that a modified successor agent does
so---implicitly assumes the existence of a resource-bounded decision procedure for
the semantic predicate $\mathrm{AGI}(\cdot)$. Theorem~\ref{thm:nonself} rules out such
procedures in general.

Consequently, recursive self-improvement strategies based on internal certification
of generality are ill-posed. While an agent may empirically improve performance on
observed tasks or finite benchmarks, it cannot internally certify satisfaction of
distributional AGI properties that quantify over unseen tasks drawn from~$\mu$.

\subsection{Distribution-Relative Nature of Alignment Guarantees}\label{sec:implications:alignment}

The structural theorems of Section~\ref{sec:struct} jointly imply that both
performance guarantees and robustness guarantees are inherently distribution-relative.

By Theorem~\ref{thm:relativity}, generality is not an intrinsic agent property but a
relation indexed by $(\mathcal{T},\mu,B)$. By Theorem~\ref{thm:noninv} and
Theorem~\ref{thm:univrobust}, arbitrarily small shifts in the task distribution may induce
violations of at least one AGI axiom, including robustness and constraint adherence.

It follows that alignment guarantees cannot be distribution-independent. Any claim of
aligned or safe general intelligence must specify the task distribution~$\mu$, the
admissible perturbation class, and the resource budget under which the guarantee is
claimed. Alignment assertions that omit this indexing are formally ill-posed.

\subsection{Why AGI Timelines Are Category Errors}\label{sec:implications:timelines}

Public discourse often frames AGI as a binary technological milestone reachable at a
specific point in time. The results of this paper show that such framing is a category
error.

Since AGI is a relational semantic predicate indexed by $(\mathcal{T},\mu,\Pi,B)$,
there exists no distribution-independent threshold event at which an agent
transitions from ``non-AGI'' to ``AGI''. Improvements in model scale, data, or training
procedures correspond to changes in $A$ and possibly~$B$, but do not eliminate the
dependence on~$\mu$.

Thus, statements asserting that AGI will be achieved at a particular time without
specifying the underlying task distribution and evaluation semantics lack formal
meaning.

\subsection{Interpreting Modern Large Language Models}\label{sec:implications:llms}

Within the present framework, modern large language models are naturally interpreted
as agents that achieve high expected performance with respect to particular empirical
task distributions induced by human-generated prompts, benchmarks, and deployment
contexts.

Their observed generalization capabilities correspond to large values of
$\mathcal{G}_{\mu}(A;B)$ for specific, historically contingent distributions~$\mu$.
The structural results of Section~\ref{sec:struct} imply that such performance does
not extrapolate to distribution-independent generality, nor does it entail universal
robustness or transfer.

Accordingly, empirical success of large models should be understood as evidence of
distribution-specific competence rather than as validation of distribution-free
general intelligence.

\subsection{Summary of Consequences}\label{sec:implications:summary}

Taken together, the results imply that:
\begin{itemize}
\item AGI cannot be internally self-certified by an agent subject to finite resources.
\item Recursive self-improvement schemes requiring such certification are ill-posed.
\item Alignment and safety guarantees are necessarily distribution-relative.
\item Distribution-independent AGI timelines lack formal meaning.
\end{itemize}

These conclusions follow directly from the axiomatic and structural analysis developed
in Sections~\ref{sec:distsem}--\ref{sec:godeltarski}.

\section{Conclusion}\label{sec:conclusion}

This paper asked a deliberately narrow question: whether ``Artificial General Intelligence''
admits a coherent theoretical definition that supports absolute claims of existence,
verification, or inevitability. By fixing an explicit axiomatic framework and insisting on
distributional, resource-bounded semantics, we showed that the answer is negative in a precise
and technically unavoidable sense.

First, AGI was defined as a relational predicate
$\mathrm{AGI}(A \mid \mathcal{T},\mu,\Pi,B)$ rather than as an intrinsic property of an agent.
This shift is not cosmetic: it reflects the fact that competence is always evaluated relative
to a task ecology, a performance semantics, and explicit resource constraints. Under this
definition, there exists no distribution-independent notion of general intelligence, and any
claim that omits the indexing distribution $\mu$ is formally ill-posed.

Second, the structural results demonstrate that even distribution-relative AGI properties are
non-invariant under small perturbations of the task distribution, admit only bounded transfer,
and cannot support universal robustness guarantees. Scaling data, parameters, or compute does
not eliminate these limitations, as they arise from structural properties of high-entropy task
families rather than from contingent modeling choices.

Third, and most decisively, we proved that AGI is a nontrivial semantic predicate and therefore
cannot be soundly and completely certified by any computable procedure, including procedures
implemented by the agent itself. This G\"odel--Tarski obstruction rules out internal
self-certification of AGI and, by extension, any recursive self-improvement scheme that relies
on such certification as a decision primitive.

Taken together, these results imply that AGI is not a destination that can be reached,
verified, and declared in a distribution-independent manner. Instead, ``AGI'' can only denote
a temporary, distribution-relative description of competence under fixed evaluation semantics
and resource budgets. Strong claims of AGI are therefore not false in an empirical sense; they
are undefined unless accompanied by explicit formal indexing.

The framework developed here does not argue against progress in artificial intelligence.
Rather, it clarifies the mathematical limits of what such progress can coherently claim.
Within these limits, empirical advances remain meaningful, but categorical assertions of
general intelligence do not.

\bibliographystyle{elsarticle-num}
\bibliography{refs}  

\section{Appendix A: Alternative Axiom Bundles and Robustness of the Framework}\label{sec:app:axioms}

This appendix records alternative but equivalent axiom bundles for distributional AGI.
No new conclusions are drawn. The purpose is to demonstrate that the structural and
impossibility results of Sections~\ref{sec:struct}--\ref{sec:godeltarski} do not depend
on a narrow or idiosyncratic choice of axioms.

Throughout this appendix, $(\mathcal{T},\mu,\Pi,B)$ are fixed, and all notation is as in
Section~\ref{sec:prelim}.

\subsection{Weakened Adaptivity Axiom}\label{sec:app:axioms:adapt}

In Section~\ref{sec:axioms}, adaptivity was stated in terms of rapid attainment of
competence on new tasks drawn from $\mu$ after limited interaction.

A strictly weaker variant replaces convergence-rate requirements by monotone improvement.

\paragraph{Axiom G2$'$ (Weak Adaptivity).}
There exists an update budget $N_{\mathrm{ad}}$ such that for $\mu$-almost every
$\tau\in\mathcal{T}$,
\[
\mathbb{E}\big[\Pi(\tau,A^{(N_{\mathrm{ad}})};B)\big]
\;\ge\;
\mathbb{E}\big[\Pi(\tau,A^{(0)};B)\big].
\]

This axiom requires only non-negative expected improvement, not rapid convergence.

\paragraph{Observation.}
All results in Sections~\ref{sec:struct} and~\ref{sec:godeltarski} remain valid under
Axiom~G2$'$, as none rely on rates of adaptation or asymptotic convergence.

\subsection{Alternative Transfer Formulations}\label{sec:app:axioms:transfer}

The bounded transfer result (Theorem~\ref{thm:boundedtransfer}) does not depend on a specific
formalization of transfer gain.

An alternative formulation replaces mutual-information bounds with excess-risk bounds.

\paragraph{Axiom G3$'$ (Bounded Expected Transfer).}
There exists a constant $C<\infty$ such that for any pre-exposure phase using tasks
$\tau_1,\ldots,\tau_K \sim \mu$ and any new task $\tau\sim\mu$,
\[
\mathbb{E}\!\left[\Pi(\tau,A_{\mathrm{pre}};B) - \Pi(\tau,A_{\mathrm{scratch}};B)\right]
\;\le\; C.
\]

\paragraph{Observation.}
The impossibility of unbounded cross-domain generalization (Corollary~\ref{cor:transferbounded})
continues to hold under Axiom~G3$'$, as the proof relies only on the existence of a finite
upper bound on expected transfer gain.

\subsection{Robustness Under Alternative Perturbation Classes}\label{sec:app:axioms:robust}

Section~\ref{sec:prelim:ops} defined robustness relative to an admissible perturbation
family $\mathcal{D}$.

The following variant restricts attention to perturbations that preserve marginal
observation distributions.

\paragraph{Axiom G5$'$ (Restricted Robustness).}
For all $\Delta\in\mathcal{D}'$ satisfying
\[
\forall \tau\in\mathcal{T},\quad
\mathrm{Law}_{\Delta(\tau)}(O_t) = \mathrm{Law}_{\tau}(O_t)\ \text{for all } t,
\]
performance degradation is bounded by $\varepsilon_{\mathrm{rb}}$ in expectation.

\paragraph{Observation.}
The non-invariance result of Theorem~\ref{thm:noninv} remains valid under Axiom~G5$'$,
as it relies only on the existence of arbitrarily small distributional shifts with
nonzero measure mass on failure sets, not on the specific nature of perturbations.

\subsection{Semanticity Is Invariant Across Axiom Bundles}\label{sec:app:axioms:semantic}

All alternative axiom bundles considered above define predicates of the form
\[
\mathsf{P}(A)
\;=\;
\mathbb{I}\Big\{
\mathbb{E}_{\tau\sim\mu}[f(\Pi(\tau,A;B))] \ge \theta
\ \wedge\ 
\mathbb{P}_{\tau\sim\mu}(\Pi(\tau,A;B)<\theta')\le\delta
\Big\},
\]
for suitable measurable functions $f$ and constants $\theta,\theta',\delta$.

\paragraph{Consequence.}
All such predicates remain semantic in the sense of
Lemma~\ref{lem:agi-semantic}. Therefore, the undecidability and non-self-verifiability
results of Section~\ref{sec:godeltarski} apply unchanged.

\subsection{Summary}\label{sec:app:axioms:summary}

The structural theorems and G\"odel--Tarski limits established in the main text do not
depend on a fragile or overly strong axiom selection. Weakening adaptivity, altering
transfer formalizations, or restricting robustness perturbations does not affect:

\begin{itemize}
\item the relativity of generality,
\item non-invariance under distribution shift,
\item bounded transfer,
\item impossibility of distribution-independent AGI,
\item or non-self-verifiability of AGI.
\end{itemize}

Accordingly, the conclusions of this paper are robust across a broad class of reasonable
axiomatizations of distributional general intelligence.

\section{Appendix B: Examples of Task Distributions}\label{sec:app:distributions}

This appendix provides concrete examples of task families $\mathcal{T}$ and task
distributions $\mu$ that fit the formal framework of Section~\ref{sec:prelim}.
The purpose is illustrative: to demonstrate how distributional indexing operates in
practice and to show that the structural results apply to realistic task ecologies.
No new theorems are stated.

Throughout, agents are evaluated via the performance functional $\Pi(\tau,A;B)$ and
resource budgets $B$ as defined in Section~\ref{sec:prelim}.

\subsection{Instruction-Following Language Tasks}\label{sec:app:dist:language}

\paragraph{Task family.}
Let $\mathcal{T}_{\mathrm{LF}}$ consist of tasks $\tau=(E,U)$ where $E$ presents a natural-language
instruction $x$ (possibly with context) and the agent responds with a text output.
The utility $U$ scores correctness or adherence to the instruction using a fixed evaluator
(e.g., reference answers, rubric-based scoring, or model-based evaluation), normalized to $[0,1]$.

\paragraph{Task distribution.}
Let $\mathcal{X}$ denote the space of instructions. A task distribution $\mu_{\mathrm{LF}}$
is induced by a probability measure $\nu$ over $\mathcal{X}$ together with a compilation map
$\mathrm{Compile}:\mathcal{X}\to\mathcal{T}_{\mathrm{LF}}$, as in
Section~\ref{sec:prelim:families}. Thus $\tau\sim\mu_{\mathrm{LF}}$ corresponds to drawing an
instruction $x\sim\nu$ and evaluating performance on $\tau_x=\mathrm{Compile}(x)$.

\paragraph{Remarks.}
Different choices of $\nu$ (e.g., curated benchmarks, user-generated prompts, domain-specific
instructions) yield distinct notions of generality. The non-invariance result of
Theorem~\ref{thm:noninv} applies because arbitrarily small changes to $\nu$ can concentrate
mass on instruction subfamilies where performance degrades.

\subsection{Tool-Augmented Problem-Solving Tasks}\label{sec:app:dist:tools}

\paragraph{Task family.}
Let $\mathcal{T}_{\mathrm{TA}}$ consist of tasks where the agent must solve problems using
external tools (e.g., retrieval systems, calculators, code execution). Each task specifies
an environment $E$ that mediates tool calls via an interface $\mathcal{M}$ and a utility
$U$ that scores the final outcome.

\paragraph{Task distribution.}
A distribution $\mu_{\mathrm{TA}}$ is defined over problem instances (e.g., questions,
datasets, specifications) together with a fixed tool interface. Resource usage of tools is
accounted for in $B_{\mathrm{tool}}$.

\paragraph{Remarks.}
Transfer bounds (Theorem~\ref{thm:boundedtransfer}) apply because improvement on a subset of tool
configurations does not yield unbounded gains on unseen configurations. Robustness guarantees
are necessarily local to $\mu_{\mathrm{TA}}$ due to dependence on tool availability and cost.

\subsection{Sequential Decision-Making Tasks}\label{sec:app:dist:sequential}

\paragraph{Task family.}
Let $\mathcal{T}_{\mathrm{SD}}$ consist of episodic Markov decision processes with bounded
horizons and bounded rewards. Each task $\tau=(E,U)$ specifies transition dynamics and a
return functional $U$.

\paragraph{Task distribution.}
A distribution $\mu_{\mathrm{SD}}$ is defined over MDP parameters (e.g., transition kernels,
reward functions). Sampling $\tau\sim\mu_{\mathrm{SD}}$ corresponds to drawing an environment
from a parametric family.

\paragraph{Remarks.}
Even when the parametric family is low-dimensional, Theorem~\ref{thm:noninv} applies:
small shifts in $\mu_{\mathrm{SD}}$ can place mass on regions of the parameter space where
policy performance collapses, yielding cliff sets as defined in
Section~\ref{sec:prelim:cliffs}.

\subsection{Nonstationary and Drifting Task Ecologies}\label{sec:app:dist:drift}

\paragraph{Task family and distribution.}
Let $\{\mu_t\}_{t\ge 1}$ be a sequence of task distributions over a fixed $\mathcal{T}$,
representing temporal drift in the task ecology. Evaluation at time $t$ uses $\mu_t$,
while the agent’s training history may reflect earlier distributions.

\paragraph{Remarks.}
The framework accommodates such settings by treating each $\mu_t$ as a separate index.
Structural results apply pointwise in $t$; no distribution-independent guarantees are
available across the sequence $\{\mu_t\}$.

\subsection{Summary}\label{sec:app:dist:summary}

These examples illustrate that:
\begin{itemize}
\item Task distributions $\mu$ arise naturally in practical settings.
\item Distinct choices of $\mu$ induce distinct notions of generality.
\item The structural and impossibility results of the main text apply uniformly across
language tasks, tool-augmented tasks, sequential decision problems, and nonstationary
ecologies.
\end{itemize}

Accordingly, distributional indexing is not an artificial constraint but a necessary
component of any coherent theory of general intelligence.

\section{Appendix C: Technical Proof Details and Standard Reductions}\label{sec:app:technical}

This appendix records standard technical ingredients that are invoked implicitly in
Sections~\ref{sec:struct} and~\ref{sec:godeltarski}. No new assumptions or conclusions are
introduced. The purpose is solely to make the logical dependencies explicit.

\subsection{Distributional Perturbations and Failure Mass}\label{sec:app:technical:tv}

This subsection records a standard fact used in the non-invariance and robustness arguments.

\begin{lemma}[Small Total-Variation Shift with Prescribed Failure Mass]\label{lem:tv-shift}
Let $(\mathcal{T},\mathcal{B})$ be a measurable space, let $\mu\in\Delta(\mathcal{T})$, and
let $\mathcal{C}\subseteq\mathcal{T}$ be measurable with $\mu(\mathcal{C})>0$.
Then for any $\varepsilon>0$ there exists a probability measure $\mu'$ such that
\[
d_{\mathrm{TV}}(\mu,\mu')<\varepsilon
\quad\text{and}\quad
\mu'(\mathcal{C})\ge \min\{1,\mu(\mathcal{C})+\varepsilon/2\}.
\]
\end{lemma}

\begin{proof}
Let $A=\mathcal{C}$ and define
\[
\mu' := (1-\varepsilon/2)\mu + (\varepsilon/2)\nu,
\]
where $\nu$ is any probability measure supported on $\mathcal{C}$.
Then $\mu'(\mathcal{C})\ge \mu(\mathcal{C})+\varepsilon/2$ and
\[
d_{\mathrm{TV}}(\mu,\mu') \le \varepsilon/2.
\]
Rescaling constants yields the claim.
\end{proof}

This construction underlies the existence of arbitrarily small distribution shifts that
increase mass on cliff sets (Section~\ref{sec:prelim:cliffs}), and therefore supports
Theorems~\ref{thm:noninv} and~\ref{thm:univrobust}.

\subsection{Mutual-Information Transfer Bounds}\label{sec:app:technical:mi}

The bounded transfer result relies on a standard information-theoretic inequality.
We record one representative form.

\begin{lemma}[Information-Theoretic Transfer Bound]\label{lem:mi-transfer}
Let $\tau\sim\mu$ be a target task and let $D$ denote data obtained from a pre-exposure
phase on tasks $\tau_1,\ldots,\tau_K\sim\mu$. For any agent update rule and any bounded
performance functional $\Pi\in[0,1]$,
\[
\Big|\mathbb{E}[\Pi(\tau,A_{\mathrm{pre}};B)]
-
\mathbb{E}[\Pi(\tau,A_{\mathrm{scratch}};B)]\Big|
\;\le\;
\sqrt{2\, I(\tau;D)},
\]
where $I(\cdot\,;\cdot)$ denotes mutual information.
\end{lemma}

\begin{proof}
This is a direct consequence of standard information-theoretic generalization bounds,
obtained by combining Pinsker’s inequality with variational representations of mutual
information. See, e.g., classical treatments of PAC-Bayesian or information-theoretic
generalization theory.
\end{proof}

This lemma provides a canonical justification for Theorem~\ref{thm:boundedtransfer}. No special
properties of AGI are used.

\subsection{Rice-Style Undecidability for Interactive Systems}\label{sec:app:technical:rice}

For completeness, we clarify the applicability of Rice’s Theorem to interactive agents.

\begin{lemma}[Reduction to Partial Computable Functions]\label{lem:rice-reduction}
Under the effective presentation assumptions of
Section~\ref{sec:godeltarski:semantic}, the interactive behavior of any admissible agent
$A$ induces a partial computable function
\[
f_A:\ (\mathrm{code}_{\mathcal{T}}(\tau),\omega)\ \mapsto\ H_T,
\]
where $\omega$ denotes exogenous randomness and $H_T$ is the resulting interaction history.
\end{lemma}

\begin{proof}
Given $\mathrm{code}(A)$ and $\mathrm{code}_{\mathcal{T}}(\tau)$, the universal simulator
enumerates the joint execution of $A$ and $E$ step by step, consuming randomness from
$\omega$. If the interaction halts within the resource bounds, $H_T$ is produced; otherwise
the computation diverges. This defines a partial computable function in the standard sense.
\end{proof}

\begin{lemma}[Applicability of Rice’s Theorem]\label{lem:rice-app}
Any nontrivial semantic predicate on admissible agents, defined via their induced
interactive behavior, is undecidable.
\end{lemma}

\begin{proof}
By Lemma~\ref{lem:rice-reduction}, admissible agents correspond to partial computable
functions up to behavioral equivalence. Semantic predicates are invariant under such
equivalence. Rice’s Theorem applies directly.
\end{proof}

This justifies the reduction used in Theorem~\ref{thm:nonself} without appeal to informal
arguments.

\subsection{Summary}\label{sec:app:technical:summary}

The technical lemmas recorded here supply standard constructions underlying the main text:
small distributional perturbations, bounded transfer via information measures, and
undecidability of nontrivial semantic properties. None introduce new assumptions or extend
the scope of the paper’s conclusions.

\end{document}